\documentclass[a4paper, 11pt]{article}

\usepackage{graphicx}
\usepackage[T1]{fontenc}
\usepackage[latin9]{inputenc}
\usepackage{amsmath}
\usepackage{amsthm}
\usepackage{amssymb}
\usepackage[unicode=true] {hyperref}
\usepackage{tikz}
\usepackage{subcaption}
\usepackage{authblk}

\newtheorem{defn}{Definition}
\newtheorem{thm}{Theorem}
\newtheorem{lem}{Lemma}
\newtheorem{cor}{Corollary}
\newtheorem{rem}{Remark}
\newtheorem{exa}{Example}

\DeclareMathOperator{\E}{E}
\DeclareMathOperator{\KL}{KL}
\DeclareMathOperator{\argmax}{argmax}

\begin{document}


\title{Probabilistic Duality for Parallel Gibbs Sampling without Graph Coloring}
\author[1]{Lars Mescheder}
\author[2]{Sebastian Nowozin}
\author[1]{Andreas Geiger}
\affil[1]{Autonomous Vision Group, MPI T\"ubingen }
\affil[2]{Microsoft Research, Cambridge}

\maketitle

\begin{abstract}
We present a new notion of probabilistic duality for random variables involving mixture distributions.
Using this notion, we show how to implement a highly-parallelizable Gibbs sampler for weakly coupled discrete pairwise graphical models with strictly positive factors that requires almost no preprocessing and is easy to implement. Moreover, we show how our method can be combined with blocking to improve mixing. Even though our method leads to inferior mixing times compared to a sequential Gibbs sampler, we argue that our method is still very useful for large dynamic networks, where factors are added and removed on a continuous basis, as it is hard to maintain a graph coloring in this setup. Similarly, our method is useful for parallelizing Gibbs sampling in graphical models that do not allow for graph colorings with a small number of colors such as densely connected graphs.
\end{abstract}

\section{Introduction}
\label{sec:intro}
Inference in general discrete graphical models is hard. Besides variational methods,
the main approach for inference in such models is through running a Markov chain
that in the limit draws samples from the true posterior distribution. 

One such Markov-Chain is given by the so-called \emph{Gibbs-sampler} that was first introduced
introduced by Geman and Geman \cite{geman1984stochastic}. In each step, one random variable is resampled given all the others,
using the conditional probability distributions in the graphical model. Under mild
hypotheses, the Gibbs sampler produces an ergodic Markov chain that
converges to the target distribution. The main appeal of the Gibbs
sampler lies in its simplicity and ease of implementation. Unfortunately,
for highly coupled random variables, mixing of this Markov chain
can be prohibitively slow. Moreover, in order to achieve
ergodicity, we have to sample one variable after another, yielding
an inherently sequential algorithm. However, with the advent of affordable
parallel computing hardware in the form of GPUs, it is desirable to
have a parallel sampling algorithm. Early attempts tried running all
the update steps in parallel \cite{geman1984stochastic}, however this update schedule
usually does not converge to the target distribution \cite{newman2007distributed}.

Another common approach to parallel Gibbs-sampling is to compute a graph coloring of the underlying graph and then perform
Gibbs-sampling blockwise \cite{gonzalez2011parallel}. However, it is not always straightforward to find an appropriate graph coloring and
it is hard to maintain such a graph coloring in a dynamic setup, i.e. when factors are added and removed on a continuous
basis.

In this paper, we show a simple method of parallelizing a Gibbs-sampler that does not
require a graph coloring. This is particularly useful for situations in which a graph
coloring is hard to obtain or the graph topology changes frequently, which requires to maintain
or to recompute the graph coloring.
\section{Related Work}
\label{sec:related}
An early attempt at parallelization for Ising models was described
by Swendsen and Wang \cite{swendsen1987nonuniversal}. Their method
was generalized to arbitrary probabilistic graphical models in \cite{barbu2005generalizing}.
However, while the Swendsen Wang-algorithm mixes fast for the Ising
model with no unary potentials, this needs not be the case for general
probabilistic graphical models \cite{higdon1998auxiliary,martens2010parallelizable}.
Higdon \cite{higdon1998auxiliary} presents a method for performing
partial Swendsen Wang-updates. However, sampling using Higdon's
method requires sampling from a coarser graphical model as a subproblem,
which Higdon tackles using conventional sampling methods. Our dualization
strategy allows to circumvent this step, so that only standard clusterwise
sampling as in \cite{swendsen1987nonuniversal} is required.

\cite{gonzalez2011parallel} describes two ways of parallelizing Gibbs
sampling in discrete Markov random fields. The first method relies
on computing graph colorings, the second on decomposing the graph
into blocks (called splashes) consisting of subgraphs with limited
tree width. Both methods are complimentary in that graph colorings
work well for loosely-coupled graphical models whereas splash sampling
works well in the strongly coupled case. However, computing
a minimal graph coloring is a NP-hard problem \cite{garey1974some} and the number
of colorings necessary depends on the graph. Moreover, it is hard to maintain
a graph coloring in a dynamic setting in which the graph topology is not constant anymore.
Our approach does not
suffer from these issues and requires almost no preprocessing. Moreover,
our approach can be combined with splash sampling. Whereas the approach
in \cite{gonzalez2011parallel} requires the splashes to be induced
subgraphs of the graphical model, our approach allows to select arbitrary
subgraphs of the graphical model as splashes, making it possible to
use splashes containing many variables.

Schmidt et al. show in \cite{schmidt2010generative} how a Gaussian
scale mixtures (GSMs) can be used for efficiently sampling from fields
of experts \cite{roth2009fields}. Our approach is very similar to
theirs, but deals with discrete graphical models. Moreover, whereas	
Schmidt et al. started with a model in a primal-dual formulation and
trained it on data, we focus on decomposing existing graphical models
and mainly use duality for inference. In fact, both techniques can
be subsumed in the framework of exponential family harmoniums \cite{welling2004exponential},
which makes it possible to deal with models consisting of both discrete
and continuous random variables.

Martens et al. \cite{martens2010parallelizable} show how to sample
from a discrete graphical model using auxiliary variables. Their approach
is similar to ours, but relies on computing the (sparse) Cholesky
decomposition of a large matrix beforehand. Our approach does not
have this issue.

A variant of our algorithm that computes expectations instead of performing sampling 
corresponds to the mean-field-algorithm
for junction tree-approximations in \cite{wiegerinck2000variational}. We can show that our algorithm
minimizes an upper bound to the true mean-field objective.
However, whereas the mean-field algorithm in \cite{wiegerinck2000variational}
has to recalibrate the tree for each update of a potential, our algorithm
updates all the potentials at once with only one run of the junction-tree-algorithm.

Schwing et al. \cite{schwing2011distributed} designed a system that
performs belief propagation in a distributed way. Our approach has
a similar goal but follows a different strategy: while Schwing et
al. achieve parallelism through a convex formulation, we augment the
probabilistic model with additional random variables.

\section{Probabilistic duality}\label{sec:probabilistic-duality}

We first define the notion of duality of random variables. Note that
this definition is very similar to the Lagrange functional of a convex
optimization problem in convex analysis.
\begin{defn}
Let $x\in\Omega_{x}$ and $\theta\in\Omega_{\theta}$ denote random
variables. We call functions $s:\Omega_{x}\rightarrow V$ and $r:\Omega_{\theta}\rightarrow V$
to a common vector space with some bilinear form link functions.
We say $x$ and $\theta$ are dual to each
other via $(s,r)$ if the joint distributions can be written as
\[
p(x,\theta)=h(x)g(\theta)\mathrm{e}^{\langle s(x),r(\theta)\rangle}
\]

with positive functions $h:\Omega_{x}\rightarrow\mathbb{R}$ and $g:\Omega_{\theta}\rightarrow\mathbb{R}$.
\end{defn}
In the language of \cite{welling2004exponential}, a dual pair of
random-variables is simply an \emph{exponential family harmonium}.
For a pair of link functions $(s,r)$ and $h:\Omega_{x}\rightarrow\mathbb{R}$,
$g:\Omega_{\theta}\rightarrow\mathbb{R}$ some real valued functions
we now define the concept of an $(s,r)$-transform.
\begin{defn}
The $(s,r)$-transforms of $h:\Omega_{x}\rightarrow\mathbb{R}$, $g:\Omega_{\theta}\rightarrow\mathbb{R}$
is defined as
\begin{align*}
H(\theta) & :=\sum_{x}h(x)\mathrm{e}^{\langle s(x),r(\theta)\rangle}\\
G(x) & :=\sum_{\theta}g(\theta)\mathrm{e}^{\langle s(x),r(\theta)\rangle}.
\end{align*}

\end{defn}
Note the resemblance to the notion\emph{ convex conjugacy}. The following
simple lemma is central to the rest of the theory:
\begin{lem}
Let be $x$ and $\theta$ two dually related random variables with
joint probability density as above. Then
\begin{align*}
p(x) & =h(x)G(x)\\
p(\theta) & =H(\theta)g(\theta)\\
p(x\mid\theta) & =\frac{h(x)}{H(\theta)}\mathrm{e}^{\langle s(x),r(\theta)\rangle}\\
p(\theta\mid x) & =\frac{g(\theta)}{G(x)}\mathrm{e}^{\langle s(x),r(\theta)\rangle}.
\end{align*}

\end{lem}
Formally, this is similar to the notion of duality in convex optimization,
we call the problem of sampling from $p(x)$ the \emph{primal} and
the corresponding problem of sampling from $p(\theta)$ the \emph{dual
sampling problem}. The primal and dual problems are linked to each
other via the conditional densities $p(x\mid\theta)$ and $p(\theta\mid x)$.
Note that both $p(x\mid\theta)$ and $p(\theta\mid x)$ are in the
exponential family.

Another view is that we represent $p(x)$ as a mixture of probability
distributions $p(x\mid\theta)$ from a specified exponential family
determined by $s(x)$ and $h(x)$. The density of the mixture parameters
is then given by $p(\theta)$.

To obtain a dual formulation of a sampling problem, we have to decompose
$p(x)$ as 
\[
p(x)=h(x)G(x)=h(x)\sum_{\theta}g(\theta)\mathrm{e}^{\langle s(x),r(\theta)\rangle}
\]
with some functions $g(\theta)$ and $h(x)$. $p(x)$ is then the
marginal of 
\[
p(x,\theta)=h(x)g(\theta)\mathrm{e}^{\langle s(x),r(\theta)\rangle}.
\]
We show how this can be done for the discrete case in Section \ref{sec:binary-pairwise-MRFs}.

Further evidence can be incorporated by replacing $s(x)$ and $h(x)$ with 
\[
\tilde{s}(x)=s(x,x_{e})\quad\text{and}\quad\tilde{h}(x,x_{e})=h(x,x_{e}).
\]

The lemma already shows a possible strategy how to sample from $p(x)$ using a simple Gibbs sampler:
first sample from $p(x\mid\theta)$, then from $p(\theta\mid x)$
and so on. As both $x$ and $\theta$ are generally  high-dimensional,
sampling from $p(x\mid\theta)$ and from $p(\theta\mid x)$ could potentially cause problems.
However, we show that $p(x\mid\theta)$ and $p(\theta\mid x)$ both
factorize in Markov random fields for appropriate choices $s$ and
$r$ which yields algorithms that are very easy to parallelize.

\section{Duality in MRFs}

We now take a closer look at sampling in Markov random fields. The
\emph{Hammersley-Clifford theorem }states that the joint probability
density of nodes in the MRF can be decomposed as
\[
p(x)=\frac{1}{Z}\prod_{i=1}^N p_{i}(x)
\]

with some probability measures $p_{i}$ and normalization constant
$Z$. In general the $p_{i}$ only depend on a small subset of the
components of $x$ (e.g. the cliques in the MRF). Assume that for
every $p_{i}$ we have a random variable $\theta_{i}$, so that $x$
and $\theta_{i}$ are dual to each other via $(s,r_{i})$ with respect
to $p_{i}$. We can then write 
\[
p_{i}(x,\theta_{i})=h_{i}(x)g_{i}(\theta_{i})\mathrm{e}^{\langle s(x),r_{i}(\theta_{i})\rangle}.
\]

\begin{thm}\label{thm:binary-pairwise-fact}
Let $\theta:=(\theta_{1},\dots,\theta_{N})$ and $r(\theta):=\sum_{i}r_{i}(\theta_{i})$,
$h(x):=\prod_{i}h_{i}(x)$ and $g(\theta):=\prod_{i}g_{i}(\theta_{i})$. 

Then $p(x)$ is the marginal of 
\[
p(x,\theta)\propto h(x)g(\theta)\mathrm{e}^{\langle s(x),r(\theta)\rangle}.
\]

Thus if $p(x,\theta)\propto h(x)g(\theta)\mathrm{e}^{\langle s(x),r(\theta)\rangle}$,
$x$ and $\theta$ are dual to each other. The marginal distribution
of $\theta$ is given by
\[
p(\theta)\propto H(\theta)\prod_{i}g_{i}(\theta).
\]
\end{thm}
\begin{proof}
We have
\begin{align*}
p(x) & \propto\prod_{i} h_{i}(x)G_{i}(x)
 = \prod_{i} \left[ h_{i}(x)\sum_{\theta_{i}}g_{i}(\theta_{i})\mathrm{e}^{\langle s(x),r_{i}(\theta_{i})\rangle} \right]\\
 & =\sum_{\theta_{1}}\dots\sum_{\theta_N}\left[ \prod_{i}h_{i}(x)\right] \left[\prod_{i}g_{i}(\theta_{i})\right] \mathrm{e}^{\langle s(x),\sum_{i}r_{i}(\theta_{i})\rangle}\\
 & =\sum_{\theta}h(x)g(\theta)\mathrm{e}^{\langle s(x),r(\theta)\rangle}.
\end{align*}
\end{proof}
\begin{cor}
$p(x\mid\theta)$ and $p(\theta\mid x)$ are given by
\begin{align*}
p(x\mid\theta) & \propto\prod_{i}h_{i}(x)\mathrm{e}^{\langle s(x),\sum_{i}r_{i}(\theta_{i})\rangle}\\
p(\theta\mid x) & \propto\prod_{i}g_{i}(\theta_{i})\mathrm{e}^{\langle s(x),\sum_{i}r_{i}(\theta_{i})\rangle}.
\end{align*}

In particular, $p(\theta\mid x)$ factorizes and if $h_{i}\equiv1$
and every component of $s(x)$ only depends on one component of $x$,
$p(x\mid\theta)$ factorizes as well. In particular, this is true for the standard choice $s(x)=x$.
\end{cor}

\subsection{Binary pairwise MRFs}

\label{sec:binary-pairwise-MRFs}

We now want to turn to the special case of a pairwise binary MRF.
It turns out that finding a dual representation is equivalent to
finding an appropriate factorization of the probability table.
\begin{thm}
Let $P$ be proportional to the probability table of two binary random
variables $x_{1}$ and $x_{2}$. Assume we are given a factorization
$P=BC^{\intercal}$ with $B,C\in\mathbb{R}^{2\times2}$, where both
$B$ and $C$ have strictly positive entries. Let 
\begin{align*}
\alpha_{1} & =\log\frac{B_{2,1}}{B_{1,1}}\\
\alpha_{2} & =\log\frac{C_{2,1}}{C_{1,1}}\\
q & =\log\frac{B_{1,2}C_{1,2}}{B_{1,1}C_{1,1}}\\
\beta_{1} & =\log\frac{B_{2,2}B_{1,1}}{B_{1,2}B_{2,1}}\\
\beta_{2} & =\log\frac{C_{2,2}C_{1,1}}{C_{1,2}C_{2,1}}.
\end{align*}
Then
\[
p(x_{1},x_{2})\propto\sum_{\theta\in\{0,1\}}h(x)g(\theta)\mathrm{e}^{\langle x,r(\theta)\rangle}
\]
with
\begin{align*}
r(\theta) & =\theta\left(\begin{array}{c}
\beta_{1}\\
\beta_{2}
\end{array}\right).\\
h(x) & =\mathrm{e}^{\alpha_{1}x_{1}}\mathrm{e}^{\alpha_{2}x_{2}}\\
g(\theta) & \mathrm{=\mathrm{e}}^{q\theta}.
\end{align*}
 \end{thm}
\begin{proof}
This follows from
\[
P=\sum_{i=1,2}\left(\begin{array}{c}
B_{1,i}\\
B_{2,i}
\end{array}\right)\left(\begin{array}{cc}
C_{1,i} & C_{2,i}\end{array}\right)
\]
after a simple calculation.
\end{proof}
We will now show how to find such a factorization:
\begin{lem}
If $P$ is symmetric and $\det P\geq0$, $P$ can be factored in the
form $P=B\,B^{\intercal}$, where
\begin{equation*}
B=\left(\begin{array}{cc}
\sqrt{p_{11}}\cos(\varphi) & \sqrt{p_{11}}\sin(\varphi)\\
\sqrt{p_{22}}\sin(\varphi) & \sqrt{p_{22}}\cos(\varphi)
\end{array}\right)
\;\text{with}\;
\varphi=\frac{\pi}{4}-\frac{1}{2}\arccos\left(\frac{p_{12}}{\sqrt{p_{11}p_{22}}}\right).
\end{equation*}
\end{lem}
\begin{proof}
$B$ is well defined because $0\leq\det$$P=p_{11}p_{22}-p_{12}^{2}$
and positive because $\varphi\in\left(0,\frac{\pi}{4}\right)$. Now
let
\[
B=\left(\begin{array}{c}
\tilde{b}_{1}^{\intercal}\\
\tilde{b}_{2}^{\intercal}
\end{array}\right).
\]
We have
\[
B\,B^{\intercal}=\left(\begin{array}{cc}
\|\tilde{b}_{1}\|^{2} & \langle\tilde{b}_{1},\tilde{b}_{2}\rangle\\
\langle\tilde{b}_{1},\tilde{b}_{2}\rangle & \|\tilde{b}_{2}\|^{2}
\end{array}\right).
\]
Due to trigonometric considerations
\begin{equation*}
BB^{\intercal}=\left(\begin{array}{cc}
p_{11} & c\sqrt{p_{11}p_{22}}\\
c\sqrt{p_{11}p_{22}} & p_{22}
\end{array}\right)
\;\;\text{with}\;\;
c=\cos\left(\frac{\pi}{2}-2\varphi\right).
\end{equation*}
This shows $B\,B^{\intercal}=P$ as required.\end{proof}
\begin{rem}
For $\varphi=\frac{\pi}{4}-\frac{1}{2}\arccos(a)$, we have
\begin{align*}
\cos(\varphi) & =\frac{1}{2}(\sqrt{1+a}+\sqrt{1-a})\\
\sin(\varphi) & =\frac{1}{2}(\sqrt{1+a}-\sqrt{1-a}).
\end{align*}
\end{rem}
\begin{lem}
For any $P$, then
\[
\left(\begin{array}{cc}
p_{12}^{-1} & 0\\
0 & p_{21}^{-1}
\end{array}\right)P
\]

is symmetric.
\end{lem}

\begin{lem}
If $\det P_{i}<0$ , then
\[
\left(\begin{array}{cc}
0 & 1\\
1 & 0
\end{array}\right)P
\]

has positive determinant.
\end{lem}
In summary, we have found a strictly positive
factorization of any strictly positive $2\times2$ matrix. Together with Theorem \ref{thm:binary-pairwise-fact} this
yields a dual representation for every binary pairwise MRF.

\subsection{General discrete MRFs}

When the variables are allowed to have multiple states, we can convert
any discrete pairwise MRF into a binary MRF using $0-1$-encoding and additional
hard-constraints that ensure that exactly one binary variable belonging
to a random variable in the original MRF has value $1$. All inference
algorithms in this paper therefore generalize to this situation.

Dualizing a $n\times m$ factor in this way introduces $nm$ auxiliary
binary random variables to the model. Note however, that no new random
variables need to be introduced for $1$-entries in the factor. For
example, for a Potts-factor of order $n$, only $n$ auxiliary binary
random variables have to be introduced per factor.

Arbitrary discrete MRFs with higher-order factors work as well, as long as we can find an
appropriate positive tensor factorization of the probability table.
Moreover, it is also possible to perform inference approximately by
fitting a mixture of Bernoulli / mixture of Dirichlet distributions
to the factors using expectation maximization.

\subsection{Swendsen-Wang and local constraints}

As it turns out, the Swendsen-Wang algorithm can be seen as a degenerate
special case of our formalism for a particular choice of $s(x)$.
Moreover, more general local constraint models can potentially be
derived from this formalism.

More explicitly, for the Ising model let 
\begin{equation*}
s(x):=(-I(x_{e_{1}}=x_{e_{2}}))_{e\in E}\label{eq:local-constraint}
\end{equation*}
where $E$ denotes the set of edges and
\[
I(x_{e_{1}}=x_{e_{2}})=\begin{cases}
0 & \text{if }x_{e_{1}}=x_{e_{2}}\\
\infty & \text{else}
\end{cases}.
\]
The Ising factor of the form
\begin{equation*}
P_{i}\propto\left(\begin{array}{cc}
1 & \mathrm{e}^{-w_{i}}\\
\mathrm{e}^{-w_{i}} & 1
\end{array}\right)=\left(\begin{array}{cc}
\mathrm{e}^{-w_{i}} & \mathrm{e}^{-w_{i}}\\
\mathrm{e}^{-w_{i}} & \mathrm{e}^{-w_{i}}
\end{array}\right)
+\left(\begin{array}{cc}
1-\mathrm{e}^{-w_{i}} & 0\\
0 & 1-\mathrm{e}^{-w_{i}}
\end{array}\right)
\end{equation*}
can then be decomposed as
\[
P_{i}(x_{e_{1}},x_{e_{2}})=\sum_{\theta_{i}\in\{0,1\}}g(\theta_{i})\mathrm{e}^{-\theta_{i}I(x_{e_{1}}=x_{e_{2}})}.
\]
where
\begin{align*}
g(0) & =\mathrm{e}^{-w_{i}}\\
g(1) & =1-\mathrm{e}^{-w_{i}}.
\end{align*}
The primal dual sampling algorithm then proceeds as follows:
\begin{align*}
p(\theta_{i}\mid x) & \propto g(\theta_{i})\mathrm{e}^{-\theta_{i}I(x_{e_{1}}=x_{e_{2}})}\\
p(x\mid\theta_{i}) & \propto\mathrm{e}^{-\sum_{i}\theta_{i}I(x_{e_{1}}=x_{e_{2}})},\nonumber 
\end{align*}
which are just the update rules for the Swendsen-Wang algorithm.

The partial Swendsen-Wang method by Higdon \cite{higdon1998auxiliary}
can be regarded as a decomposition of the form
\[
P_{i}\propto\left(\begin{array}{cc}
1-\alpha & \mathrm{e}^{-w_{i}}\\
\mathrm{e}^{-w_{i}} & 1-\alpha
\end{array}\right)+\left(\begin{array}{cc}
\alpha & 0\\
0 & \alpha
\end{array}\right).
\]
This leads to the method described by Higdon, where we are left with
sampling from coarser Ising-model. By applying a factorization as
in Section \ref{sec:binary-pairwise-MRFs} to the first term enables
us to circumvent this step, so that all clusters can be sampled independently
(the latent variables $\theta$ then have $3$ different states).

Similarly, the generalization of the Swendson-Wang algorithm in \cite{barbu2005generalizing}
can be regarded as a multiplicative decomposition of the form
\[
P_{i}\propto\left(\begin{array}{cc}
\mathrm{e}^{w_{i}} & 1\\
1 & \mathrm{e}^{w_{i}}
\end{array}\right)\star\tilde{P}_{i},
\]
where we used the $\star$-operator to indicate componentwise multiplication. Applying the
decomposition above to the first factor, yields (a variant of) the
method in \cite{barbu2005generalizing}. We can use our method to
further decompose $\tilde{P}_{i}$, allowing to update all clusters
in parallel.

\section{Inference}

\subsection{Sampling}

\label{sec:inference-sampling}

Having a primal-dual representation of $p(x)$ of the form
\[
p(x,\theta)\propto h(x)g(\theta)\mathrm{e}^{\langle s(x),r(\theta)\rangle}
\]
we can sample from $p(x,\theta)$ (and thereby from $p(x)$) by blockwise
Gibbs-sampling, i.e.
\begin{align*}
x^{(t+1)} & \sim p(x\mid\theta^{(t)})\propto h(x)\mathrm{e}^{\langle s(x),r(\theta^{(t)})\rangle}\\
\theta^{(t+1)} & \sim p(\theta\mid x^{(t+1)})\propto g(\theta)\mathrm{e}^{\langle s(x^{(t+1)}),r(\theta)\rangle}.
\end{align*}
For discrete pairwise MRFs with the dual representations as described
above, both distributions factor, so that sampling can be done in
parallel, e.g. on the GPU. Effectively, we have converted our model
to a restricted Boltzmann machine. 

\subsection{Estimation of the log-partition-function}
The logarithm normalization constant of an unnormalized probability distribution, the so-called \emph{log-partition-function}, 
is important for model-selection and related tasks. In this section, we provide a simple estimator for this quantity that can be used for any
dual pair of random variables.

The following defines an unbiased estimator for the partition function $Z$:
\[
  V(x, \theta) = \frac{\tilde p(x)\tilde p(\theta)}{\tilde p(x, \theta)} = Z \frac{p(x)p(\theta)}{p(x, \theta)},
\]
where $\tilde p(x)$ and $\tilde p(\theta)$ are the unnormalized probability distributions.
Indeed, we have
\[
 \E \left[ V(x, \theta) \right] = Z \int \frac{ p(x) p(\theta)}{ p(x, \theta)} p(x, \theta) \mathrm d x \mathrm d \theta
 = Z.
\]
Written in terms of $G$ and $H$, $V(x, \theta)$ can be written as
\[
 V(x, \theta) = G(x)H(\theta)\mathrm e^{-\langle s(x), r(\theta) \rangle}.
\]
Note that
\[
 \E \left[ -\log V(x, \theta) \right] - \left[ -\log \E \left[ V(x, \theta) \right] \right]
 = \mathbb I(x, \theta),
\]
where $\mathbb I(x, \theta)$ is the mutual information between $x$ and $\theta$. This is a measure 
for the uncertainty of $V(x, \theta)$, as it can be interpeted as a generalized variance for the convex
function $x \mapsto -\log(x)$. 

This also implies that the expectation of
\[
 \log V(x, \theta)
\]
yields a lower bound to the log-partition function. In practice, $V(x, \theta)$ has too much variance to be useful. Therefore,
we estimate the expectation of $\log V(x, \theta)$, which yields a lower bound to the log-partition function.

\begin{exa}\label{exa:log-partition-SW}
For the Swendsen-Wang-duality for the Ising-model, we have $h(x) = 1$, $G(x) = \prod_e P_e(x_{e_1}, x_{e_2})$ and therefore
\[
 H(\theta) = \sum_x \prod_e \mathrm e^{-\theta_e I(x_{e_1}= x_{e_2})} = 2^{C(\theta)},
\]
where $C(\theta)$ is the number of clusters defined by $\theta$. Therefore
\[
 \log V(x, \theta) = \log 2 \cdot C(\theta) + \log \tilde p(x),
\]
where $\tilde p(x)$ is the unnormalized distribution of the Ising-model.

\end{exa}

A natural question to ask is how this estimator relates to the estimate obtained by running naive mean-field on the primal distribution $p(x)$ only. The
following negative result shows that in most cases the estimate obtained by mean-field approximations is preferable:
\begin{lem}
We have
\begin{equation}\label{eq:information-inequality}
  \mathbb I(x, \theta) = \E_\theta \KL(p(x \mid \theta), p(x)) \geq \min_\xi \KL (p(x \mid \xi), p(x)),
\end{equation}
where we define $p(x \mid \xi) \propto h(x)\exp(\langle s(x), \xi \rangle)$.
\end{lem}
\begin{proof}
 The equality in \eqref{eq:information-inequality} can be obtained by a straightforward calculation. The inequality is a simple application of the fact that the expectation over $\theta$ is
 always bigger than the minimum over $\theta$.
\end{proof}
Note, however, that it is not always straightforward to find $\xi$ that minimizes $\KL (p(x \mid \xi), p(x))$. This is for example the case for the Swendsen-Wang-representation from Example \ref{exa:log-partition-SW}.

\subsection{MAP- and mean-field inference}
The concept of probabilistic duality that we introduced in Section \ref{sec:probabilistic-duality} is also useful to derive parallel MAP- and mean-field inference algorithms.

By applying EM to $p(x,\theta)$ we can also compute local MAP-assignments
to $p(x)$ in parallel. The updates read
\begin{align*}
x^{(t+1)} & =\argmax_{x}h(x)\mathrm{e}^{\langle s(x),\xi^{(t)}\rangle}\\
\xi^{(t+1)} & =\E\left(r(\theta)\mid x^{(t+1)}\right).
\end{align*}

Similar, we can compute mean-field assignments to $p(x,\theta)$ using the updates
\begin{align*}
\eta^{(t+1)} & =\E\left(s(x)\mid\xi^{(t)}\right)\\
\xi^{(t+1)} & =\E\left(r(\theta)\mid\eta^{(t+1)}\right),
\end{align*}
where the expectations are taken over the distributions
\begin{align*}
p(x\mid\xi) & \propto h(x)\mathrm{e}^{\langle s(x),\xi\rangle}\\
p(\theta\mid\eta) & \propto g(\theta)\mathrm{e}^{\langle\eta,r(\theta)\rangle}.
\end{align*}
Note that these updates have the advantage over ICM and standard naive
mean field that they can again run in parallel and still have convergence
guarantees.

Using this algorithm, we can show that we minimize an upper bound to the true mean-field objective $\KL(p(x\mid \xi), p(x))$:
\begin{lem}\label{lem:pdmeanfield-inequality}
 We have
 \begin{equation}
  \min_\eta  \KL \left(p(x\mid \xi)p(\theta, \eta), p(x, \theta) \right) \geq \KL(p(x \mid \xi), p(x)).
 \end{equation}
\end{lem}
\begin{proof}
 We have
\begin{multline*}
\KL (p(x\mid \xi)p(\theta, \eta), p(x, \theta))
= \E_{x\mid \xi} \E_{\theta \mid \eta} \left( -\log \frac{p(x, \theta)}{p(x \mid \xi) p(\theta \mid \eta)} \right) \\
\geq \E_{x\mid \xi}  \left( -\log \E_{\theta \mid \eta} \frac{p(x, \theta)}{p(x \mid \xi) p(\theta \mid \eta)} \right) 
= \E_{x\mid \xi}  \left( -\log \frac{p(x)}{p(x \mid \xi)} \right)  \\
= \KL( p(x \mid \xi), p(x)).
\end{multline*}

\end{proof}
Lemma \ref{lem:pdmeanfield-inequality} implies that traditional mean-field updates are still preferable to the ones from our method. Indeed, in 
practice we found that our method can lead to poor approximations in presence of many factors. However, it is still possible to first run our fast parallel algorithm and
then fine-tune the result using traditional mean-field updates.  

\subsection{Blocking}

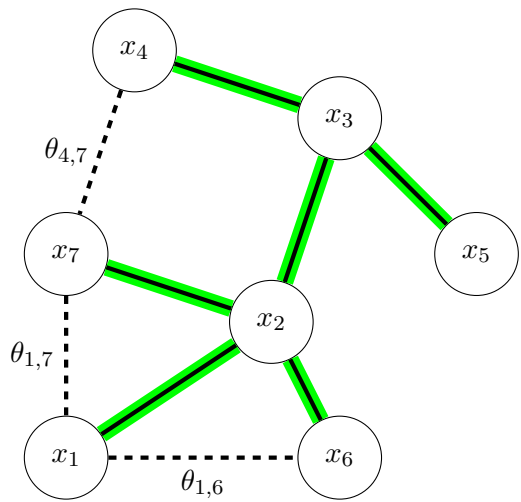
\begin{figure}
\noindent \begin{centering}
 \begin{tikzpicture}[scale=0.9]
\tikzstyle{highlighted}=[preaction={
draw,green,-,
double=green,
double distance=2\pgflinewidth,
}]

\tikzstyle{edge}=[ultra thick]
\tikzstyle{vertex}=[circle,draw,minimum width=1.1cm]

 \node[vertex] (x1)  at (-2, -2) {$x_1$} ;
 \node[vertex] (x2) at (1, 0) {$x_2$} ;
 \node[vertex] (x3) at (2, 3) {$x_3$} ;
 \node[vertex] (x4) at (-1, 4) {$x_4$} ;
 \node[vertex] (x5) at (4, 1) {$x_5$} ;
 \node[vertex] (x6) at (2, -2) {$x_6$} ;
 \node[vertex] (x7) at (-2, 1) {$x_7$} ;

\draw(x1) edge [edge, highlighted] (x2);
\draw(x2) edge [edge, highlighted] (x7);
\draw(x2) edge [edge, highlighted] (x3);
\draw(x2) edge [edge, highlighted] (x6);
\draw(x3) edge [edge, highlighted] (x5);
\draw(x3) edge [edge, highlighted] (x4);
\draw(x4) edge [edge, dashed] node[left] {$\theta_{4,7}$} (x7) ;
\draw(x7) edge [edge, dashed] node[left]{$\theta_{1, 7}$}(x1);
\draw(x1) edge [edge, dashed] node[below]{$\theta_{1, 6}$}(x6);

\end{tikzpicture}
\par\end{centering}

\caption{Example of blocking using a tree as a subgraph. Only the dual variables
on the dashed edges are used. Sampling then alternatingly samples
all $\theta_{e}$ variables and all the $x$ variables on the tree.
Similarly, the EM algorithm alternatingly maximizes all $x$ over
the tree using max-product-belief-propagation and takes conditional
expectations over the $\theta_{e}$ variables to update the unary
potentials. Our tree-mean-field algorithm does the same, except that
it uses sum-product belief propagation.}\label{fig:blocking}
\end{figure}

Gibbs sampling as described in Section \ref{sec:inference-sampling}
can still be prohibitively slow in presence of strongly correlated
random variables. Similarly, EM and mean-field updates tend to get
stuck in local optima in this situation. This problem also occurs
for standard Gibbs sampling, ICM and naive mean-field updates. A common
way out is to introduce blocking, e.g. as in \cite{gonzalez2011parallel}
for Gibbs sampling.

Unfortunately, blocking is only possible with respect to induced subgraphs
for traditional algorithms. As it turns out, our primal-dual decomposition
allows to perform blocking with respect to arbitrary subgraphs. This
is illustrated in Figure \ref{fig:blocking}. The idea is to decompose
the dual variables $\theta$ into two subsets $\theta_{0}$ and $\theta_{1}$,
so that $p(x,\theta_{0}\mid\theta_{1}$) is tractable. This is the
case, when $p(x\mid\theta_{1})$ is tractable, because then
\[
p(x,\theta_{0}\mid\theta_{1})=p(\theta_{0}\mid x)p(x\mid\theta_{1}).
\]

Note that $p(x\mid\theta_{1})$ is tractable, if the graph obtained
by removing all the factors belonging to $\theta_{1}$ has low tree-width. 

For blocked Gibbs sampling, we sample in each step $p(x\mid\theta_{1}^{(t)})$
and then $p(\theta_{1}^{(t)}\mid x^{(t+1)})$. As a variation of
this process, we can vary the decomposition of $\theta$ into $\theta_{0}$
and $\theta_{1}$ in each step. 

When we perform max-product-belief propagation for $x$ and take expectations
for $\theta_{1}^{(t)}$ we obtain a new inference algorithm for MAP-inference.
Note that in each step, we maximize over \emph{all} $x$ variables
at once. Similarly, we obtain a probabilistic inference algorithm
by changing the max-product-belief propagation step for $x$ by sum-product-belief
propagation.

Note also that both the EM, as well as the mean-field algorithm are
guaranteed to increase the objective function in each step.

In this framework, the standard sequential Gibbs sampler can also be interpreted as 
a blocked Gibbs sampler, where blocking is performed with respect to one primal and all
neighboring dual variables. Unfortunately, as blocking generally improves
mixing of a Gibbs chain \cite{gonzalez2011parallel}, this implies that the standard sequential Gibbs sampler
has better mixing properties than the parallel primal-dual sampling algorithm. Still, the primal dual formulation
allows for more flexible blocking schemes, potentially making it possible to improve on the mixing properties of the standard sequential
Gibbs sampler in some situations.

\section{Experimental Results}
\label{sec:results}
\begin{figure*}[h]
    \centering
    \begin{subfigure}[t]{0.5\textwidth}
        \centering
	\includegraphics[width=\textwidth]{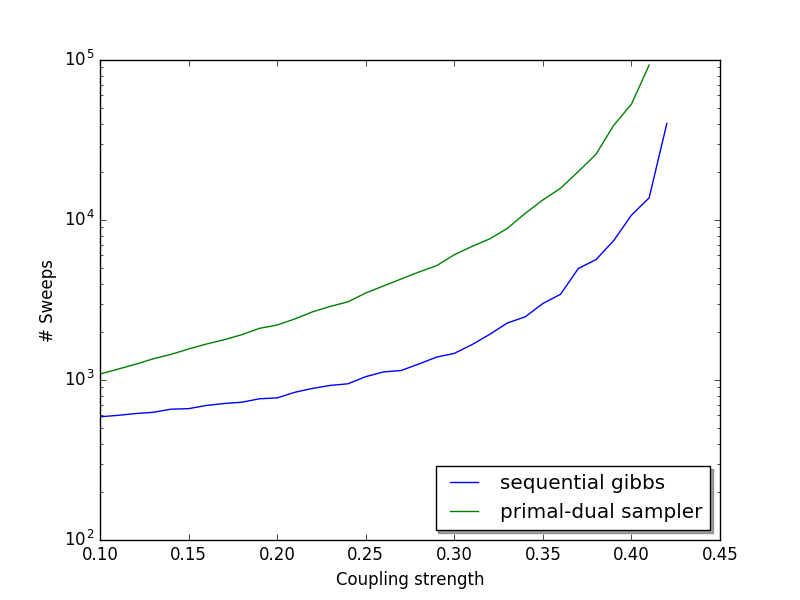}
	\caption{For the Ising model.}
	 \label{mixing-ising}
    \end{subfigure}%
    ~ 
    \begin{subfigure}[t]{0.5\textwidth}
        \centering
	\includegraphics[width=\textwidth]{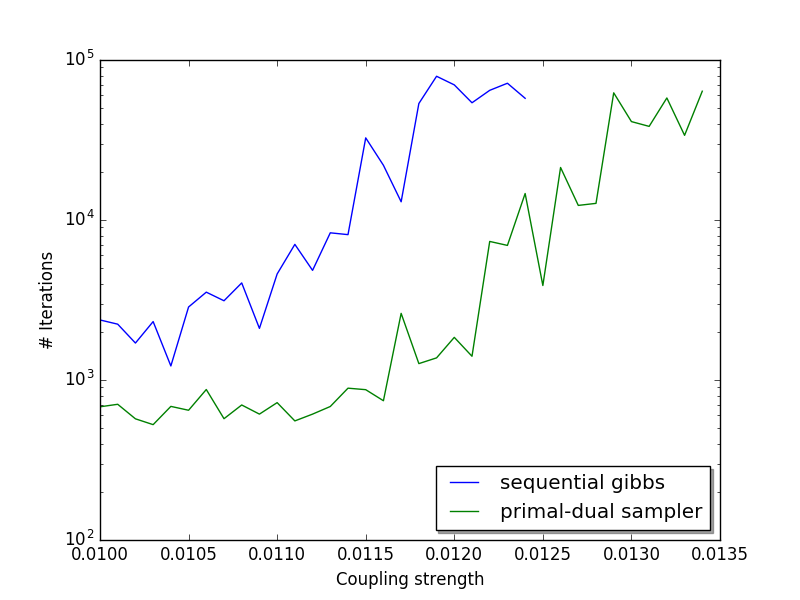}
        \caption{For the fully connected Ising model.}
	\label{mixing-fullising}
    \end{subfigure}
    \caption{Number of iterations needed to achieve a PSRF below $1.01$.}
    
 \centering
\end{figure*}

We tested our method on $3$ synthetic graphical models.
The first model consists of a $50 \times 50$-Ising grid and coupling strenghts ranging from $\beta = 0.1$ to $\beta = 0.5$. Even though the Ising grid is two-colorable and it is therefore trivial to implement a parallel Gibbs sampler in  this setting, this is no longer possible when the Graph topology is dynamic, i.e. we remove and add factors from time to time. Maintaining a coloring in this setting is itself a hard problem. 
The second model is given by a random graph with $N = 1000$ variables
and $F = k\,N$ factors, where
$k \in \{2, 4, 8, 16, 32, 64\}$. Both the unitary and pairwise log-potentials were sampled from a normal distribution with mean $0$ and a standard deviation of $1$.
The last model consists of a fully connected Ising model of $N=100$ variables and 
coupling strengths ranging from $\beta = 0.01$ to $\beta = 0.015$. Note that for such models, there is an algorithm that computes the partition function and the marginals in polynomial time \cite{flach2013class}. However, this is not longer the case, when the potentials have varying coupling strengths.

For all these models we compute the potential scale reduction factor (PSRF) for both a sequential Gibbs sampler and our primal-dual-sampler by running $10$ Markov chains in parallel. From the PSRF we compute an estimate of the mixing time of the Markov chain by taking the first index, so that the PSRF remains below some specified threshhold afterwards.

The result for the Ising grid is shown in Figure \ref{mixing-ising}. For both the primal-dual sampler and the sequential Gibbs sampler, we plotted the number of seeps over the whole grid to achieve a PSRF below $1.01$. As expected, both the sequential Gibbs sampler and the primal-dual sampler mix slower as we increase the coupling strength. Moreover, even though  the primal-dual sampler mixes slower than the sequential Gibbs sampler, in our experiments the ratio of the mixing times was between $2$ and $7$ for all the coupling strength. Therefore, even though a Gibbs sampler based on a two coloring is preferable in the static setting, our primal-dual sampler becomes a viable alternative in the dynamic setting.

Similar results where obtained for the random graphs. As expected, mixing of the primal dual sampler became worse as the number of factors per vertix increases. While our primal-dual-sampler can
be an interesting alternative when the factor-to-vertex ratio is low (e.g. $k \approx 2$), we do not recommend our method for models with many more factors than variables if these factors are
not very weak.

The result for the fully connected Ising-model is shown in Figure \ref{mixing-fullising}. As there is no coloring available for a fully connected graphical model, we compare the number of full sweeps of our primal-dual sampler against the number of single-site updates of the sequential Gibbs sampler. We see that our method leads to improved mixing in this setting. 

%
%
%
%
%
%
%

\section{Conclusion}
\label{sec:conclusion}

We have introduced a new concept of duality to random variales and
showed its usefulness in performing inference in probabilistic graphical
models. In particular, we demonstrated how to obtain a highly-parallizing Gibbs sampler.
Even though this parallel Gibbs sampler has inferior mixing properties compared to the 
sequential Gibbs sampler, we believe that it can still be very useful in settings, where 
a good graph coloring is hard to obtain or the graph topology changes frequently.
Possible extensions of
our approach include good algorithms for selecting appropriate subgraphs
for blocking. Moreover, as primal-dual representations are not unique,
we believe that further progress can be made by deriving new decompositions.
Another line of research is to generalize our ideas to higher order
factors, both exactly and in approximate ways. We believe that this
is possible and allows to apply the methods in this paper to arbitrary
discrete graphical models.

\section*{Acknowledgements}
This work was supported by Microsoft Research through its PhD Scholarship Programme.
\bibliographystyle{plain}
\bibliography{bib/bibliography} 

\end{document}